\documentclass[journal]{IEEEtran}
\usepackage[font=footnotesize]{caption}
\usepackage[font=footnotesize]{subcaption}
\usepackage{graphicx}
\usepackage{amsmath,amsthm,amssymb}
\theoremstyle{definition}
\newtheorem{definition}{Definition}
\usepackage{theoremref}
\usepackage{multirow}
\usepackage{comment}
\newcommand\subplotsize{0.4}

\begin{document}
\title{Improved Compact Genetic Algorithms with Efficient Caching}
\author{Prasanta~Dutta
        and Anirban~Mukhopadhyay
\thanks{Prasanta Dutta is with the Physics and Applied Mathematics Unit, Indian Statistical Institute, Kolkata 700108, West Bengal, India (e-mail: prasantadutta@isical.ac.in)

Anirban Mukhopadhyay is with the Department of Computer Science and Engineering, University of Kalyani, Kalyani 741235, West Bengal, India (e-mail: anirban@klyuniv.ac.in)}}
\maketitle
\begin{abstract}
Compact Genetic Algorithms (cGAs) are condensed variants of classical Genetic Algorithms (GAs) that use a probability vector representation of the population instead of the complete population. cGAs have been shown to significantly reduce the number of function evaluations required while producing outcomes similar to those of classical GAs. However, cGAs have a tendency to repeatedly generate the same chromosomes as they approach convergence, resulting in unnecessary evaluations of identical chromosomes. This article introduces the concept of caching in cGAs as a means of avoiding redundant evaluations of the same chromosomes. Our proposed approach operates equivalently to cGAs, but enhances the algorithm's time efficiency by reducing the number of function evaluations. We also present a data structure for efficient cache maintenance to ensure low overhead. The proposed caching approach has an asymptotically constant time complexity on average. The proposed method further generalizes the caching mechanism with higher selection pressure for elitism-based cGAs. We conduct a rigorous analysis based on experiments on benchmark optimization problems using two well-known cache replacement strategies. The results demonstrate that caching significantly reduces the number of function evaluations required while maintaining the same level of performance accuracy.
\end{abstract}
\section{Introduction}
\IEEEPARstart{G}{enetic} Algorithms (GAs) are one of the most well-known population-based meta-heuristic optimization algorithms. GAs are inspired by the process of natural selection and Darwinian evolution theory \cite{goldberg1989genetic}. GAs usually start with an initial random population of solutions and try to improve it by repeatedly executing steps like selection, crossover, and mutation. GAs have been widely used for solving optimization problems in a broad range of applications like business, engineering, and science \cite{goldberg1989genetic,back1997handbook}. 

GAs are often criticized for their memory and time requirements to store and evaluate the whole population of chromosomes. Various studies have been conducted to improve these aspects without affecting the performance. As one of the initial attempts in this direction, Baluja {\em et al.} introduced the Population-Based Incremental Learning (PBIL) algorithm, which represents the whole population as a probability vector \cite{baluja1994population,baluja1995removing}. Initially, the probability vector values are set to 0.5, and the values change toward 0 or 1 as the search progresses. This probability vector is used to generate a new set of individuals. As the search progresses, the probability vector is updated and is moved toward the fittest of the generated solutions. 

PBIL and its variants are very efficient in terms of accuracy, but they require sufficient space and time to generate the optimal result. Many real-world applications require optimization in limited memory and power environments. Compact Evolutionary Algorithms (cEA) have been designed to cope with this problem. The cEA belongs to the class of Estimation of Distributed Algorithms (EDA), which does not process and store the entire population \cite{larranaga2001estimation}. Instead, it represents the population statistically, and thus it requires less memory. Harik \textit{et al.} proposed a Compact Genetic Algorithm (cGA), which is the first cEA that represents the population as a probability distribution and is operationally equivalent to the order-one behavior of simple GA with uniform crossover \cite{harik1999compact}. In cGA, at every iteration, two chromosomes are generated based on the probability vector, and the probability vector is updated based on the fittest chromosome. The cGA can also simulate higher selection pressures by generating more than two chromosomes at each iteration and using appropriate selection methods \cite{harik1999compact}. An analysis of the convergence properties of cGA by using Markov chains is given in \cite{rastegar2006step}. The Real-valued Compact Genetic Algorithm (rcGA) was introduced in \cite{mininno2008real}, which implemented the compact logic in the real-valued domain. Ahn \textit{et al.} introduced two elitism-based cGAs, viz., Persistent Elitist Compact Genetic Algorithm (pe-cGA) and Nonpersistent Elitist Compact Genetic Algorithm (ne-cGA) \cite{ahn2003elitism}. In pe-cGA, the winner chromosome is stored as an elite chromosome, which significantly reduces the number of function evaluations. However, it leads to premature convergence and does not always produce the optimal result. On the contrary, ne-cGA overcomes this problem by changing the elite chromosome after a predefined number of iterations if the same chromosome is chosen as elite multiple times. Harik \textit{et al.} proposed an Extended Compact Genetic Algorithm (ecGA) for solving deceptive problems by combining a greedy Marginal Product Model (MPM) search algorithm with a Minimum Description Length (MDL) search model \cite{harik1999linkage,harik2006linkage}. A theoretical analysis of ecGA can be found in \cite{sastry2000extended}. The real-coded version of ecGA has also been proposed in \cite{fossati2007simple,lanzi2008real}. 

It has been observed that a larger population size generally leads to better solutions in Genetic Algorithms (GAs), albeit with increased computational cost (fitness function evaluations) and memory requirements. Several studies have explored the incorporation of caching in GAs to avoid the re-evaluation of previously generated solutions. Santos \textit{et al.} developed a caching method for storing partial results in GAs, employing a divide-and-conquer algorithm for fitness function computation \cite{santos2000cache,santos2001effective}. This approach helps reduce the number of fitness function evaluations, particularly in scenarios where fitness function computation is time-consuming, and has demonstrated efficacy in applications such as protein folding. Additionally, a caching GA for spectral breakpoint matching was proposed by Mohr \textit{et al.} \cite{mohr2004caching}. The study demonstrated time savings by caching objective function values from previous evaluations. Yuen \textit{et al.} proposed a non-revisiting GA with a binary space partitioning tree archive design \cite{yuen2008genetic, su2020non}, which has shown superior performance compared to standard evolutionary algorithms. The continuous non-revisiting genetic algorithm (cNrGA) enhances search performance through search history and parameter-less adaptive mutation \cite{lou2016non}. Two pruning mechanisms, least recently used pruning and random pruning, maintain cNrGA's performance, outperforming real-coded genetic algorithms and standard particle swarm optimization. Various variants of GAs designed for faster performance have also been proposed, involving modifications to population initialization \cite{gao2008study}, identifying patterns in chromosome genes \cite{tseng2008fast}, replacing crossover with a repair phase \cite{yuan2010fast}, and adjusting the mutation rate \cite{doerr2017fast}. These modified algorithms, termed Fast Genetic Algorithms (FGAs), exhibit faster performance but are often limited to specific optimization problems. In contrast, the Fitness Value Memoization Genetic Algorithm (FVMGA), introduced by Girsang \textit{et al.} \cite{girsang2022fast}, caches previously calculated fitness values to reduce redundant computation. In comparison with traditional GAs for optimizing Long Short-Term Memory (LSTM) hyperparameters in time-series forecasting, FVMGA demonstrated a 291\% faster computation rate. It's worth noting that while FVMGA and Non-revisiting GAs focus on avoiding revisiting duplicate chromosomes to enhance GA performance, no existing approach focuses on caching techniques for cGAs to the best of our knowledge. Our proposed approach in this article aims to improve the performance of cGA and elitism-based cGAs.

As cGAs work with a single probability vector as a representation of the whole population, the chromosomes generated from it have less probability of being distinct, especially toward convergence. Therefore caching the fitness values is expected to reduce the number of function evaluations significantly. However, to the best of our knowledge, no study in the literature has yet been conducted to explore the incorporation of caching of fitness values in cGAs. Motivated by this, in this article, we introduce the concept of caching in cGA, pe-cGA, and ne-cGA with minimum system overhead. The algorithms have been modified to incorporate caching to avoid repeated evaluation of the same chromosomes. A data structure for efficient maintenance of the cache is also proposed. The effect of different cache replacement policies and cache sizes have been rigorously studied based on several performance metrics through experiments on two benchmark optimization problems. The experimental results demonstrate how the proposed cache-based cGAs (cache-based cGA, cache-based pe-cGA, cache-based ne-cGA) reduce the number of fitness function evaluations without compromising the accuracy. 

The rest of the article is organized as follows. The detailed description of cGA, pe-cGA, and ne-cGA are given in Section~\ref{background}. The proposed cache-based cGA is described in Section~\ref{proposed} along with the implementation techniques and analysis. Section~\ref{metrics} describes the metrics for the performance evaluation. The experimental results are reported in Section~\ref{results}. The article is concluded in Section~\ref{conclude}.

\section{Variants of cGA} \label{background}
In this section, we describe the basic cGA algorithm and its variants. 

\subsection{Compact Genetic Algorithm (cGA)\label{sec:cGA}}
\begin{figure}
\centering
\includegraphics[width=\linewidth]{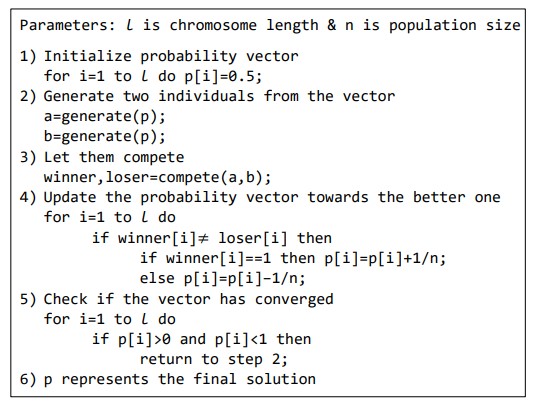}
\caption{Pseudocode of cGA}
\label{fig:cGA}
\end{figure}

The original cGA algorithm was proposed in \cite{harik1999compact}. The cGA algorithm represents the population statistically using a binary Probability Vector ($PV$) of length $l$, where $l$ is the chromosome length. Initially, the value at each gene position of $PV$ is set to 0.5. As the algorithm proceeds, two random individuals are generated from $PV$ and their fitness values are calculated. Based on the fitness values, the winner chromosome influences $PV$ based on the virtual population size, $n$. More specifically, if the winner solution has $1$ in its $i$th gene position, and the loser has $0$ in that position, then the $i$th position of $PV$ is incremented by $\frac{1}{n}$. Similarly, if the winner's $i$th gene position is $0$ and the loser's corresponding gene position is $1$, then $PV$ is decremented by $\frac{1}{n}$. If both the winner and the loser chromosomes have the same values in the $i$th gene position, then the $i$th position of $PV$ is not updated. The algorithm continues until the value in each gene position of $PV$ is either 0 or 1. For clarity, the pseudocode of cGA is given in Fig.~\ref{fig:cGA}.

\begin{figure}
\centering
\includegraphics[width=\linewidth]{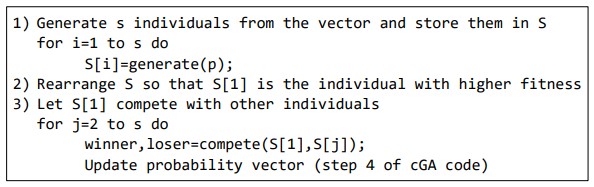}
\caption{Modification of cGA that implements tournament selection of size $s$. This would replace steps 2-4 of the cGA code}
\label{fig:cGA_t}
\end{figure}
To simulate higher selection pressure, a modified cGA with a tournament size of $s$ was proposed in \cite{harik1999compact}. Unlike the original cGA, here $s$ individuals are generated from $PV$ in every iteration. After that, the best individual with respect to the fitness function competes with other ($s-1$) individuals and updates $PV$ accordingly. The rest of the procedure is the same as cGA. The pseudocode of the selection process for this algorithm is given in Fig.~\ref{fig:cGA_t}.

\begin{figure}
\centering
\includegraphics[width=\linewidth]{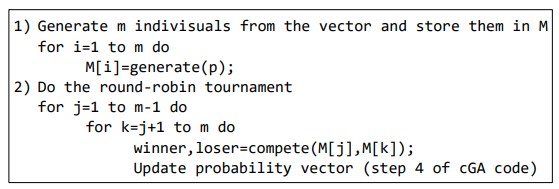}
\caption{Modification of cGA that implements a round-robin tournament. This would replace steps 2-4 of the cGA code}
\label{fig:cGA_rrt}
\end{figure}
The article \cite{harik1999compact} further proposed a round-robin tournament selection using cGA. Here instead of generating two individuals, $m$ individuals are generated at every iteration. Then every individual competes with every other individual and $PV$ is updated accordingly. The pseudocode of the selection process in this algorithm is given in Fig.~\ref{fig:cGA_rrt}.

\subsection{Elitism-Based cGA}

\begin{figure}
\centering
\includegraphics[width=\linewidth]{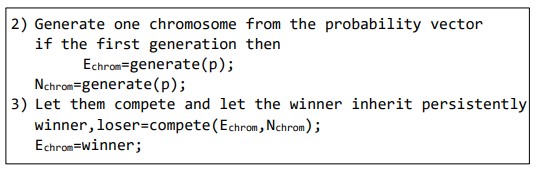}
\caption{Modification of cGA that realizes the pe-cGA. Here only the modified steps are given, the rest of the steps are same as cGA}
\label{fig:pe_cGA}
\end{figure}
Elitism-based cGA, namely persistent elitist compact genetic algorithm (pe-cGA) and nonpersistent elitist compact genetic algorithm (ne-cGA) were proposed in \cite{ahn2003elitism}. Both the algorithms are based on the idea of cGA, but they showed significantly better performance than cGA. In pe-cGA, initially, two chromosomes are generated and the winner chromosome is set as the elite chromosome. In subsequent iterations, only one chromosome is generated. This newly generated chromosome then competes with the elite chromosome and the winner chromosome is then chosen as elite. $PV$ is updated based on the winner chromosome at every iteration. The rest of the procedure is the same as cGA. The pseudocode of this algorithm is given in Fig.~\ref{fig:pe_cGA}.

\begin{figure}
\centering
\includegraphics[width=\linewidth]{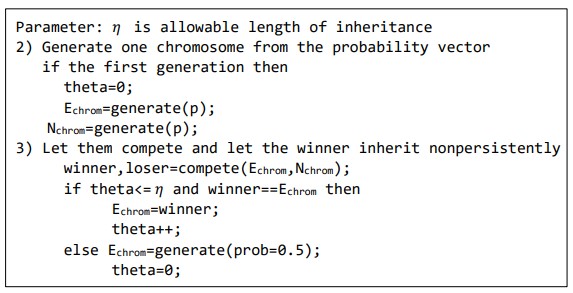}
\caption{Modification of cGA that realizes the ne-cGA. Here only the modified steps are given, the rest of the steps are same as cGA}
\label{fig:ne_cGA}
\end{figure}

In the ne-cGA scheme, the only difference is that if the elite chromosome is still not replaced after the $\eta$ number of comparisons, the elite chromosome is replaced by a newly generated chromosome irrespective of its fitness value. The pseudocode of ne-cGA is given in Fig.~\ref{fig:ne_cGA}. The preference of pe-cGA and ne-cGA depends on the problems. 

Both pe-cGA and ne-cGA have a better convergence rate, but they can not outperform the performance of cGA in terms of accuracy.

\section{Proposed cGA with Efficient Caching}\label{proposed}

The cGA and elitism-based cGA algorithms maintain a probability vector to represent the whole population in a compact form. Thus the memory requirement for these algorithms is $l \times log_2{(n+1)}$ \cite{harik1999compact}, where $l$ is the length of a chromosome (number of genes in a chromosome) and $n$ is the population size (number of chromosomes in the population). As new chromosomes are generated based on the probability vector, the probability of getting distinct chromosomes is relatively low compared to a simple GA. If we maintain a \textit{cache} and efficiently update it, we can reduce the number of fitness function evaluations significantly. It will be helpful where fitness function evaluation is a costly procedure in terms of time and space. In this section, we describe our proposed algorithm. Before describing our actual algorithm, some terms are defined as follows- 

{\em Cache capacity} denotes the maximum number of elements that can be accommodated inside the \textit{cache}. {\em Cache length} denotes the number of elements inside the {\em cache} at a particular time instance. {\em Fitness function} denotes the function based on which the fitness value of a chromosome is evaluated. {\em Fitness value} indicates the evaluated value of a chromosome based on a fitness function.

\subsection{Algorithm}

\begin{figure}
\centering
\includegraphics[width=\linewidth]{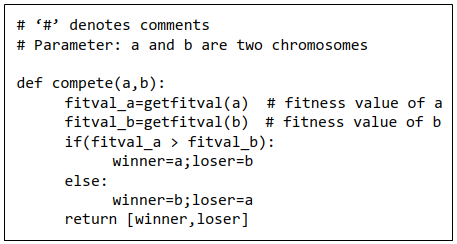}
\caption{Illustration of \textit{compete} function}
\label{fig:compete}
\end{figure}

\begin{figure}
\centering
\includegraphics[width=\linewidth]{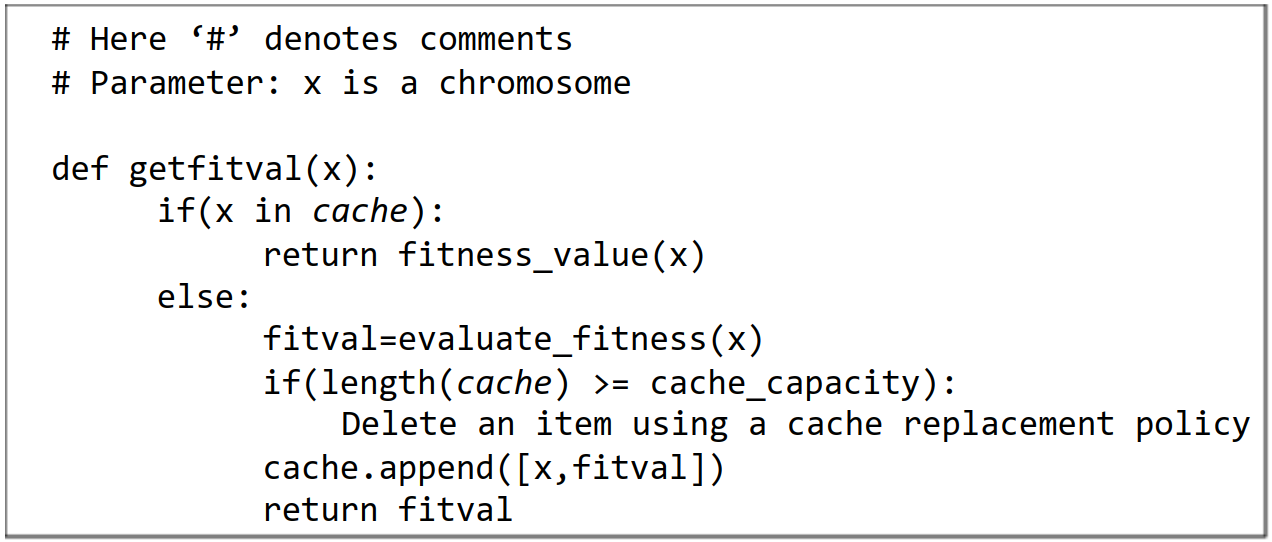}
\caption{This function returns the fitness value of a chromosome using a cache}
\label{fig:getfitval}
\end{figure}

In our proposed method, a \textit{cache} of fixed capacity is maintained. Each cache line contains a chromosome of length $l$ and the chromosome's fitness value based on a fitness function. When the fitness function of a chromosome $x$ has to be evaluated, the cache is checked first. If $x$ is found inside the cache, i.e. a cache hit occurs, the fitness value of $x$ is brought from the cache, instead of fitness function evaluation. If $x$ is not found inside the cache, i.e. a cache miss occurs, the fitness function of $x$ is evaluated and the fitness value of $x$ is stored inside the cache to the position corresponding to $x$, provided that the cache is not full. If the cache is full, some cache replacement policies are applied to accommodate $x$ along with its fitness value. In this work, two cache replacement policies, FIFO (First In First Out) and LRU (Least Recently Used) are used. In FIFO, the oldest chromosome is deleted from the cache to accommodate $x$. In LRU, the chromosome that is least recently accessed is deleted to accommodate $x$. The rest of the procedure is the same as cGA or elitism-based cGA (pe-cGA or ne-cGA). The pseudocode of our proposed method is given in Fig.~\ref{fig:getfitval}. Also, the illustration of the \textit{compete} function is given in Fig.~\ref{fig:compete}.

From the above discussion, it is clear that each element of the cache is a \textit{key-value} pair, where a chromosome is the \textit{key}, and its fitness value is the corresponding \textit{value}. In the next subsection, we discuss an efficient data structure for maintaining the cache. 

\subsection{Data Structure for the Cache}

\begin{figure}
\centering
\includegraphics[width=\linewidth]{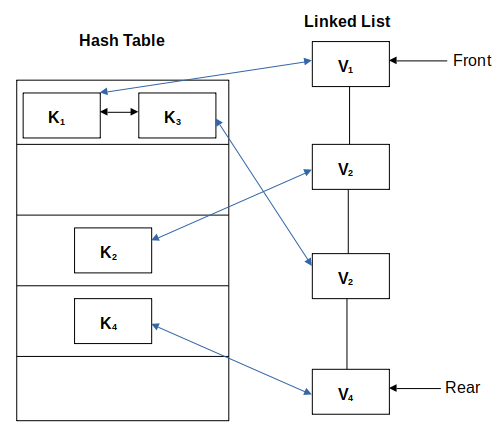}
\caption{Logical diagram of the cache implementation with cache length 4}
\label{fig:cache}
\end{figure}

A data structure is needed for the cache using which the fitness value of a chromosome can be fetched in constant time (on average) with a reasonably small space complexity. Hence we propose a hash table with a linked list as a suitable data structure to implement the cache. The hash table is used for faster lookup, whereas the linked list is used to maintain a priority queue needed at the time of cache replacement when the cache is full. Here the keys (chromosomes) are inserted (hashed) into the hash table with links that point to the nodes containing their respective fitness values (blue links in Fig.~\ref{fig:cache}). The blue links have to be two-directional because when a fitness value has to be deleted (at the time of cache replacement), the corresponding key field (chromosome) has also to be deleted. The logical diagram of the data structure is depicted in Fig.~\ref{fig:cache}. 

Here two keys may be hashed into the same location of the hash table. Chaining is used as a collision resolution technique. The direction of the arrows in the linked list is not displayed as it depends on the cache replacement strategy. In the FIFO cache replacement strategy, the new elements (nodes) are inserted into the rear position of the linked list, and the elements (nodes) are deleted from the front position when the cache is full. Hence a singly-linked list is sufficient for the implementation since no node has to be deleted from the intermediate position. In the case of LRU cache replacement, the strategy is similar to FIFO. However, here when a cache hit occurs, the node may be deleted (it is not deleted in the case of the last node) from the intermediate position and inserted into the rear of the linked list. Hence, in this case, a doubly-linked list is needed since the link of the predecessor node of the node to be deleted has to be updated in constant time. If two different chromosomes (keys) have the same fitness value (value), they will be treated separately, as shown in Fig.~\ref{fig:cache}.  Though $K_2$ and $K_3$ have the same value $V_2$, they are inserted separately in the linked list (Since $K_2$ and $K_3$ are different).  

\subsection{Analysis}
In this subsection, we discuss the reason for choosing FIFO and LRU as the cache replacement strategies. We also discuss the space and time complexities and the accuracy of our proposed method. 

\subsubsection{FIFO and LRU as the Cache Replacement Strategies}
In cGA, the chromosomes are generated based on the probability vector. Hence the probability of generating distinct chromosomes is relatively low for a particular probability vector instance (snapshot of the probability vector at a specific time). As the search progresses, the probability vector also changes and generates new chromosomes. The chance of getting distinct chromosomes becomes lower when the probability vector values converge toward 0 or 1. For this reason, FIFO is suggested as one of the cache replacement strategies in our proposed method. It keeps local copies of the most recently generated chromosomes, thus ensuring more cache hits. LRU generalizes the concept by storing the copies of the most recently used chromosomes inside the cache.

\subsubsection{Accuracy and Convergence}
The accuracy and the rate of convergence of cGA with and without caching are the same since the actual algorithm of cGA is not modified here. In our proposed method, the number of fitness function evaluations is reduced by introducing a cache. The situation is also the same for the elitism-based cGA (pe-cGA, ne-cGA) with and without cache.

\subsubsection{Time Complexity}
When a new chromosome is generated based on the probability vector, the cache is searched first for its fitness value. As the hash table is used for searching, it takes asymptotically constant time complexity on average, considering the number of hash table slots is at least proportional to the number of elements in the table \cite{cormen2009introduction}. However, the worst-case time complexity for searching the hash table is asymptotically linear. This situation can be avoided by choosing a good hash function, which distributes the keys uniformly. Some hash functions may work well for a particular type of input but perform poorly for another type. This can be avoided if we choose hash functions from the universal set of hash functions \cite{carter1979universal}. A dynamic perfect hashing scheme can also be used that provides asymptotically constant time complexity for lookup in the worst case and asymptotically constant amortized expected time for insertion and deletion operations \cite{dietzfelbinger1994dynamic}. 

When there's a cache miss and the cache is at full capacity, an element is removed from the cache, and the new element is added to the end. This operation exhibits asymptotically constant time complexity in the worst-case scenario. In the FIFO (First-In-First-Out) cache replacement strategy, when the cache reaches its limit, the first element in the linked list is removed. This operation also has an asymptotically constant time complexity in the worst-case scenario. In contrast, with the LRU (Least Recently Used) cache replacement strategy, when a cache hit occurs, the corresponding element is relocated to the end of the linked list. Because a doubly linked list is employed, the time complexity remains asymptotically constant in the worst-case scenario. 

From the above discussion, it is clear that the overall time complexity of cache maintenance is bounded by the time complexity of the Hash Table maintenance. 

\subsubsection{Space Complexity}
Let the cache capacity be $L_c$. Then the hash table contains at most $L_c$ number of entries, and each entry contains a key (chromosome) of length $l$ (say). Hence the space needed for the hash table is $O(L_c \times l)$. The linked list contains at most $L_c$ fitness values. If each fitness value can be represented by $b$ bits, then the cache's overall space complexity is $O(L_c \times l) + O(L_c \times b)$ (The space complexity of a single pointer is O(1), hence nullified asymptotically). As $b$ is constant (bits required to store floating-point numbers), the complexity becomes $O(L_c \times l) + O(L_c)$, which is $O(L_c \times l)$ according to the property of $O$ notation. 

If $L_c$ is proportional to $\log_2(n+1)$, i.e., $L_c = \alpha \times \log_2(n+1) + \beta$ (where $\alpha$ and $\beta$ are constants), the space complexity for our cache would be $O(l \times \log_2(n+1))$. It ensures that the space complexity remains the same asymptotically in the overall process with or without cache.

\section{Metrics for Performance Evaluation}\label{metrics}
As the convergence and accuracy are the same for the algorithms with and without cache, our experimental results are compared in terms of the total number of fitness function evaluations. Before going further, here two more metrics, viz., \textit{Hitratio} and \textit{Speedup} are defined based on which the performance of cache-based and corresponding non-cache-based algorithms are compared. 

\begin{definition}
{\em Hitratio} is the ratio of the number of cache hits to the total number of cache accesses. It is defined as follows:
\begin{equation}
Hitratio = \frac{h}{h + m},
\label{hr}
\end{equation} 
where $h$ and $m$ are the total number of cache hits and total number of cache misses, respectively. It can be represented in percentage as well by multiplying it by 100 ({\em Hitratio(\%)}).
\end{definition}

\begin{definition}
{\em Speedup} is the ratio of the number of function evaluations in non-cache-based algorithms to that in the corresponding cache-based algorithms. It is defined as per the following equation.
\begin{equation}
Speedup = \frac{neval}{neval_{cache}},
\end{equation}
where $neval$ and $neval_{cache}$ are the total number of fitness function evaluations without cache and with cache, respectively. {\em Speedup} implies how well the cache-based algorithms perform over the equivalent non-cache-based algorithms.
\end{definition}

In our proposed method, the fitness function for a chromosome is evaluated only in case of a cache miss. So the total number of fitness function evaluations is equal to the number of cache misses, $m$. Hence, in the case of the corresponding non-cache-based algorithm, the total number of fitness function evaluations should be equal to $h+m$. Hence it follows that the number of fitness function evaluations with cache is $neval_{cache} = m$, whereas the
number of fitness function evaluations without cache is $neval = h+m$.

Therefore the {\em Speedup} equation can be written as:
\begin{equation}
Speedup = \frac{neval}{neval_{cache}} = \frac{h+m}{m} = 1 + \frac{h}{m}.
\label{su}
\end{equation}
Hence the speedup increases with the increase in the number of cache hits, or equivalently with the decrease in the number of cache misses.

\newtheorem{lemma}{Lemma}
\begin{lemma}\thlabel{hit}
{\em Hitratio(\%)} is the measure of the percentage of reduction in the number of fitness function evaluations. 
\end{lemma}
\begin{proof}
As discussed above, the number of fitness function evaluations for the cache-based algorithms is $m$, and that for corresponding non-cache-based algorithms is $h+m$. Hence the percentage of reduction in the number of fitness function evaluations $reduct\_eval(\%)$ can be given as
\begin{align*}
reduct\_eval(\%) &= \frac{(h+m)-m}{h+m} \times 100 \\
&= \frac{h}{h+m} \times 100\\
&= Hitratio(\%) \\
\end{align*}
\end{proof}
Hence, {\em Hitratio(\%)} is an important metric that not only denotes the percentage of cache hits but also denotes the percentage of reduction in the number of fitness function evaluations.

\begin{lemma}\thlabel{hit_speedup}
$Hitratio = 1 - \frac{1}{Speedup}$ and $Speedup = \frac{1}{1-Hitratio}$.
\end{lemma}
\begin{proof}
It follows from Eqn.~\ref{hr} and \ref{su} that
\begin{align*}
Hitratio &= \frac{h}{h+m}, \\
		&= \frac{(h+m)-m}{h+m}, \\
		&= 1 - \frac{m}{h+m}, \\
		&= 1 - \frac{1}{Speedup}.
\end{align*}
Alternatively, 
\begin{align*}
Speedup &= \frac{1}{1 - Hitratio}.
\end{align*}
\end{proof}

The above lemma shows that when {\em Speedup} increases, {\em Hitratio} also increases, and vice-versa.

\newtheorem{corollary}{Corollary}
\begin{corollary}
$1 \leq Speedup \leq \infty$.
\end{corollary}
\begin{proof}
$$0 \leq Hitratio \leq 1,$$
From Lemma~\ref{hit_speedup}
$$ \therefore 0 \leq 1 - \frac{1}{Speedup} \leq 1.$$
From the above we get,
\begin{align*}
0 \leq 1 - \frac{1}{Speedup},
~\mbox{or},~ \frac{1}{Speedup} \leq 1,
~\mbox{or},~ Speedup \geq 1.
\end{align*}
and
\begin{align*}
1 - \frac{1}{Speedup} \leq 1,
~\mbox{or},~ 0 \leq \frac{1}{Speedup},
~\mbox{or},~ Speedup \leq \infty
\end{align*}
\end{proof}

\section{Experimental Results}\label{results}

\begin{figure*}
     \centering
     \begin{subfigure}{\subplotsize\linewidth}
         \includegraphics[width=\linewidth]{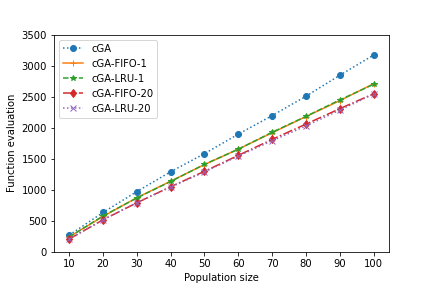}
         \caption{100-bit onemax problem}
         \label{fig:onemax_cGA_eval_n}
     \end{subfigure}
     \begin{subfigure}{\subplotsize\linewidth}
         \includegraphics[width=\linewidth]{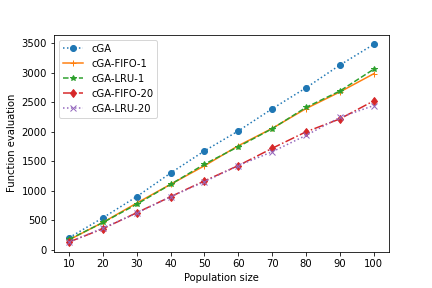}
         \caption{30-bit binary integer problem}
         \label{fig:binint_cGA_eval_n}
     \end{subfigure}
     \begin{subfigure}{\subplotsize\linewidth}
         \includegraphics[width=\linewidth]{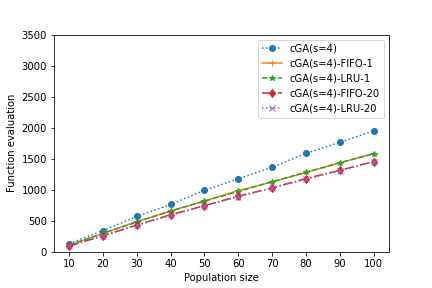}
         \caption{100-bit onemax problem}
         \label{fig:onemax_cGAt4_eval_n}
     \end{subfigure}
     \begin{subfigure}{\subplotsize\linewidth}
         \includegraphics[width=\linewidth]{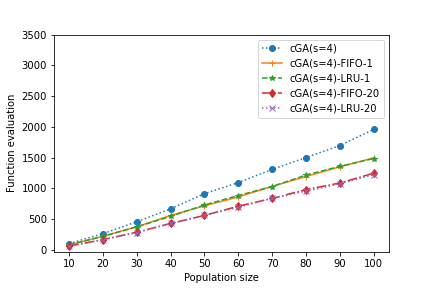}
         \caption{30-bit binary integer problem}
         \label{fig:binint_cGAt4_eval_n}
     \end{subfigure}
     \begin{subfigure}{\subplotsize\linewidth}
         \includegraphics[width=\linewidth]{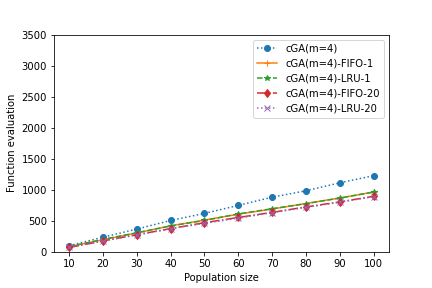}
         \caption{100-bit onemax problem}
         \label{fig:onemax_cGArrt4_eval_n}
     \end{subfigure}
     \begin{subfigure}{\subplotsize\linewidth}
         \includegraphics[width=\linewidth]{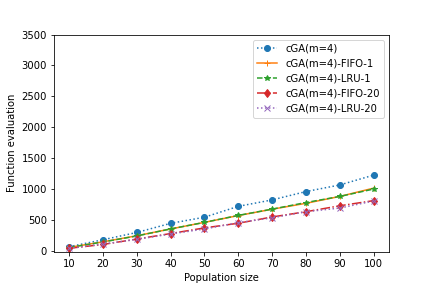}
         \caption{30-bit binary integer}
         \label{fig:binint_cGArrt4_eval_n}
     \end{subfigure}
     \caption{The plots illustrate the number of fitness function evaluations by cGA and cGA with cache for the 100-bit onemax problem and 30-bit binary integer problem with cache capacities of 1 and 20. The algorithms were executed for population sizes varying from 10 to 100 with an increment of 10}
        \label{fig:cGA_n}
\end{figure*}

\begin{figure*}
     \centering
     \begin{subfigure}{\subplotsize\linewidth}
         \includegraphics[width=\linewidth]{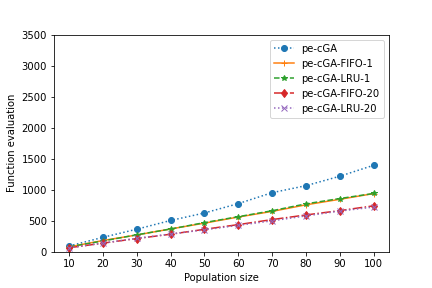}
         \caption{100-bit onemax problem}
         \label{fig:onemax_pecGA_eval_n}
     \end{subfigure}
     \begin{subfigure}{\subplotsize\linewidth}
         \includegraphics[width=\linewidth]{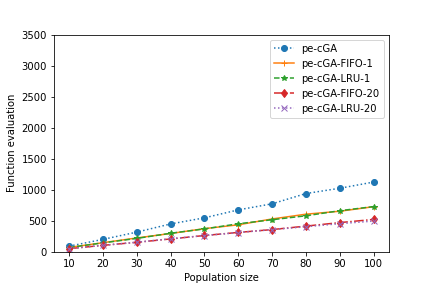}
         \caption{30-bit binary integer problem}
         \label{fig:binint_pecGA_eval_n}
     \end{subfigure}
     \begin{subfigure}{\subplotsize\linewidth}
         \includegraphics[width=\linewidth]{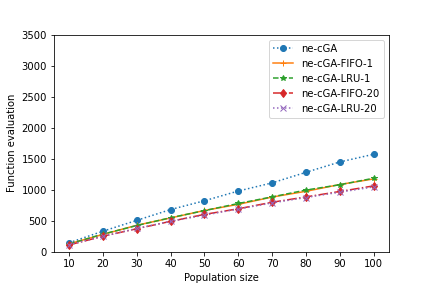}
         \caption{100-bit onemax problem}
         \label{fig:onemax_necGA_eval_n}
     \end{subfigure}
     \begin{subfigure}{\subplotsize\linewidth}
         \includegraphics[width=\linewidth]{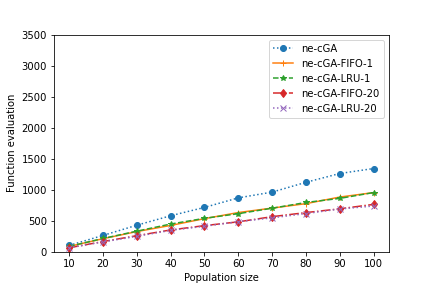}
         \caption{30-bit binary integer problem}
         \label{fig:binint_necGA_eval_n}
     \end{subfigure}
     \caption{The plots illustrate the number of fitness function evaluations by elitism-based cGA and elitism-based cGA with cache for the 100-bit onemax problem and 30-bit binary integer problem with cache capacities of 1 and 20. The algorithms were executed for population sizes varying from 10 to 100 with an increment of 10}
        \label{fig:elite_cGA_n}
\end{figure*}

\begin{figure*}
     \centering
     \begin{subfigure}{\subplotsize\linewidth}
         \includegraphics[width=\linewidth]{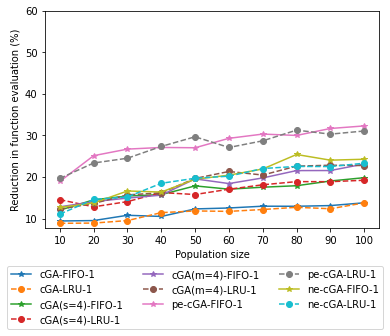}
         \caption{100-bit onemax problem}
         \label{fig:onemax_hit_c1_n}
     \end{subfigure}
     \begin{subfigure}{\subplotsize\linewidth}
         \includegraphics[width=\linewidth]{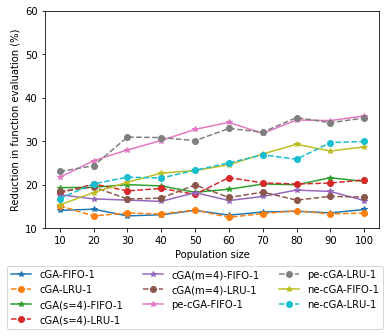}
         \caption{30-bit binary integer problem}
         \label{fig:binint_hit_c1_n}
     \end{subfigure}
     \begin{subfigure}{\subplotsize\linewidth}
         \includegraphics[width=\linewidth]{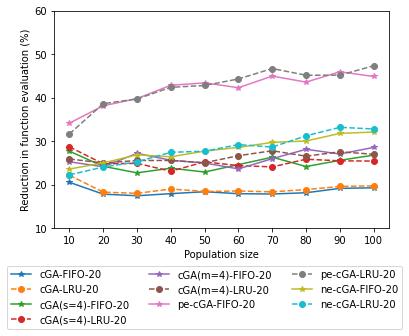}
         \caption{100-bit onemax problem}
         \label{fig:onemax_hit_c20_n}
     \end{subfigure}
     \begin{subfigure}{\subplotsize\linewidth}
         \includegraphics[width=\linewidth]{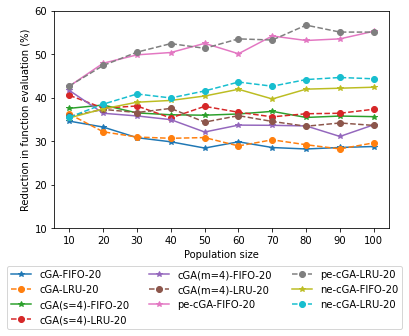}
         \caption{30-bit binary integer problem}
         \label{fig:binint_hit_c20_n}
     \end{subfigure}
     \caption{The plots illustrate the percentage of reduction in fitness function evaluation due to the use of cache by cGA and elitism-based cGA for the 100-bit onemax problem and 30-bit binary integer problem with cache capacity of 1 and 20. The algorithms were executed for population sizes varying from 10 to 100 with an increment of 10}
        \label{fig:hit_n}
\end{figure*}

\begin{table*}
\centering
\caption{{\em Speedup} achieved in terms of the number of fitness function evaluations due to the use of cache by cGA and elitism-based cGA with respect to different population sizes}
\label{tab:speedup_n}
\resizebox{\textwidth}{!}{%
\begin{tabular}{|c|l|c|l|l|l|l|l|l|l|l|l|l|l|}
\hline
\multicolumn{3}{|c|}{Population Size =\textgreater{}} & \multicolumn{1}{c|}{10} & \multicolumn{1}{c|}{20} & \multicolumn{1}{c|}{30} & \multicolumn{1}{c|}{40} & \multicolumn{1}{c|}{50} & \multicolumn{1}{c|}{60} & \multicolumn{1}{c|}{70} & \multicolumn{1}{c|}{80} & \multicolumn{1}{c|}{90} & \multicolumn{1}{c|}{100} & Average \\ \hline
\multirow{20}{*}{\begin{tabular}[c]{@{}c@{}}Onemax \\ Problem\end{tabular}} & \multirow{10}{*}{\begin{tabular}[c]{@{}l@{}}Cache\\ Length\\ 1\end{tabular}} & cGA-FIFO-1 & \multicolumn{1}{c|}{1.106839} & \multicolumn{1}{c|}{1.10677} & \multicolumn{1}{c|}{1.123163} & \multicolumn{1}{c|}{1.119271} & \multicolumn{1}{c|}{1.142201} & \multicolumn{1}{c|}{1.145342} & \multicolumn{1}{c|}{1.151185} & \multicolumn{1}{c|}{1.151527} & \multicolumn{1}{c|}{1.152234} & \multicolumn{1}{c|}{1.161735} & 1.136027 \\ \cline{3-14} 
 &  & cGA-LRU-1 & \multicolumn{1}{c|}{1.09933} & \multicolumn{1}{c|}{1.099522} & \multicolumn{1}{c|}{1.106384} & \multicolumn{1}{c|}{1.130474} & \multicolumn{1}{c|}{1.136164} & \multicolumn{1}{c|}{1.134485} & \multicolumn{1}{c|}{1.140136} & \multicolumn{1}{c|}{1.147985} & \multicolumn{1}{c|}{1.142184} & \multicolumn{1}{c|}{1.161717} & 1.129838 \\ \cline{3-14} 
 &  & cGA(s=4)-FIFO-1 & \multicolumn{1}{c|}{1.146322} & \multicolumn{1}{c|}{1.174632} & \multicolumn{1}{c|}{1.188585} & \multicolumn{1}{c|}{1.191442} & \multicolumn{1}{c|}{1.223785} & \multicolumn{1}{c|}{1.212112} & \multicolumn{1}{c|}{1.215992} & \multicolumn{1}{c|}{1.22324} & \multicolumn{1}{c|}{1.238697} & \multicolumn{1}{c|}{1.25246} & 1.206727 \\ \cline{3-14} 
 &  & cGA(s=4)-LRU-1 & \multicolumn{1}{c|}{1.179128} & \multicolumn{1}{c|}{1.151283} & \multicolumn{1}{c|}{1.168803} & \multicolumn{1}{c|}{1.200772} & \multicolumn{1}{c|}{1.190794} & \multicolumn{1}{c|}{1.208921} & \multicolumn{1}{c|}{1.226716} & \multicolumn{1}{c|}{1.238037} & \multicolumn{1}{c|}{1.236} & \multicolumn{1}{c|}{1.240729} & 1.204118 \\ \cline{3-14} 
 &  & cGA(m=4)-FIFO-1 & 1.15798 & 1.171884 & 1.184903 & 1.193004 & 1.25353 & 1.235705 & 1.253358 & 1.284025 & 1.285236 & 1.310531 & 1.233016 \\ \cline{3-14} 
 &  & cGA(m=4)-LRU-1 & 1.147453 & 1.165698 & 1.194453 & 1.204318 & 1.259923 & 1.286538 & 1.268854 & 1.31177 & 1.30819 & 1.307925 & 1.245512 \\ \cline{3-14} 
 &  & pe-cGA-FIFO-1 & 1.247907 & 1.349061 & 1.388299 & 1.381977 & 1.388011 & 1.433262 & 1.449858 & 1.440612 & 1.487133 & 1.492519 & 1.405864 \\ \cline{3-14} 
 &  & pe-cGA-LRU-1 & 1.26255 & 1.323708 & 1.338192 & 1.392388 & 1.435047 & 1.381796 & 1.412682 & 1.467338 & 1.448383 & 1.462212 & 1.39243 \\ \cline{3-14} 
 &  & ne-cGA-FIFO-1 & 1.152543 & 1.165994 & 1.204509 & 1.203014 & 1.247509 & 1.26728 & 1.289592 & 1.349951 & 1.323853 & 1.330928 & 1.253517 \\ \cline{3-14} 
 &  & ne-cGA-LRU-1 & 1.129321 & 1.177163 & 1.185311 & 1.23371 & 1.254536 & 1.264112 & 1.292448 & 1.300415 & 1.297133 & 1.30918 & 1.244333 \\ \cline{2-14} 
 & \multirow{10}{*}{\begin{tabular}[c]{@{}l@{}}Cache\\ Length\\ 20\end{tabular}} & cGA-FIFO-20 & \multicolumn{1}{c|}{1.263752} & \multicolumn{1}{c|}{1.220316} & \multicolumn{1}{c|}{1.213912} & \multicolumn{1}{c|}{1.221164} & \multicolumn{1}{c|}{1.2276} & \multicolumn{1}{c|}{1.221705} & \multicolumn{1}{c|}{1.21899} & \multicolumn{1}{c|}{1.223251} & \multicolumn{1}{c|}{1.239207} & \multicolumn{1}{c|}{1.240605} & 1.22905 \\ \cline{3-14} 
 &  & cGA-LRU-20 & \multicolumn{1}{c|}{1.291489} & \multicolumn{1}{c|}{1.225683} & \multicolumn{1}{c|}{1.222075} & \multicolumn{1}{c|}{1.237722} & \multicolumn{1}{c|}{1.228923} & \multicolumn{1}{c|}{1.228977} & \multicolumn{1}{c|}{1.226533} & \multicolumn{1}{c|}{1.235726} & \multicolumn{1}{c|}{1.246045} & \multicolumn{1}{c|}{1.247598} & 1.239077 \\ \cline{3-14} 
 &  & cGA(s=4)-FIFO-20 & \multicolumn{1}{c|}{1.398257} & \multicolumn{1}{c|}{1.328254} & \multicolumn{1}{c|}{1.299434} & \multicolumn{1}{c|}{1.317232} & \multicolumn{1}{c|}{1.301606} & \multicolumn{1}{c|}{1.3347} & \multicolumn{1}{c|}{1.360498} & \multicolumn{1}{c|}{1.323036} & \multicolumn{1}{c|}{1.350091} & \multicolumn{1}{c|}{1.373344} & 1.338645 \\ \cline{3-14} 
 &  & cGA(s=4)-LRU-20 & \multicolumn{1}{c|}{1.419126} & \multicolumn{1}{c|}{1.342134} & \multicolumn{1}{c|}{1.33974} & \multicolumn{1}{c|}{1.305332} & \multicolumn{1}{c|}{1.343535} & \multicolumn{1}{c|}{1.326543} & \multicolumn{1}{c|}{1.320451} & \multicolumn{1}{c|}{1.357406} & \multicolumn{1}{c|}{1.346646} & \multicolumn{1}{c|}{1.346258} & 1.344717 \\ \cline{3-14} 
 &  & cGA(m=4)-FIFO-20 & 1.350915 & 1.33039 & 1.392382 & 1.358086 & 1.347451 & 1.317575 & 1.363336 & 1.406442 & 1.384858 & 1.413594 & 1.366503 \\ \cline{3-14} 
 &  & cGA(m=4)-LRU-20 & 1.372264 & 1.349103 & 1.358501 & 1.35065 & 1.345961 & 1.375539 & 1.395011 & 1.372171 & 1.390402 & 1.380012 & 1.368961 \\ \cline{3-14} 
 &  & pe-cGA-FIFO-20 & 1.552531 & 1.650639 & 1.683767 & 1.786602 & 1.808992 & 1.758137 & 1.838419 & 1.801079 & 1.892916 & 1.840972 & 1.761405 \\ \cline{3-14} 
 &  & pe-cGA-LRU-20 & 1.487651 & 1.660126 & 1.676022 & 1.760365 & 1.771821 & 1.82682 & 1.919083 & 1.842634 & 1.855574 & 1.939549 & 1.773965 \\ \cline{3-14} 
 &  & ne-cGA-FIFO-20 & 1.319353 & 1.340889 & 1.380391 & 1.368152 & 1.391112 & 1.406689 & 1.43053 & 1.435305 & 1.476289 & 1.481144 & 1.402985 \\ \cline{3-14} 
 &  & ne-cGA-LRU-20 & 1.297142 & 1.327561 & 1.345399 & 1.386069 & 1.390752 & 1.418793 & 1.409061 & 1.462645 & 1.505725 & 1.502514 & 1.404566 \\ \hline
\multirow{20}{*}{\begin{tabular}[c]{@{}c@{}}Binary \\ Integer\\ Problem\end{tabular}} & \multirow{10}{*}{\begin{tabular}[c]{@{}l@{}}Cache \\ Length\\ 1\end{tabular}} & cGA-FIFO-1 & 1.16979 & 1.172266 & 1.150451 & 1.152673 & 1.166284 & 1.150981 & 1.16108 & 1.162932 & 1.157729 & 1.168369 & 1.161256 \\ \cline{3-14} 
 &  & cGA-LRU-1 & 1.184691 & 1.151391 & 1.158543 & 1.155135 & 1.167223 & 1.14505 & 1.15614 & 1.164868 & 1.154117 & 1.157258 & 1.159442 \\ \cline{3-14} 
 &  & cGA(s=4)-FIFO-1 & 1.260282 & 1.25177 & 1.259705 & 1.253052 & 1.228391 & 1.240303 & 1.257315 & 1.257037 & 1.281902 & 1.265296 & 1.255505 \\ \cline{3-14} 
 &  & cGA(s=4)-LRU-1 & 1.230771 & 1.266367 & 1.234947 & 1.245142 & 1.221043 & 1.284248 & 1.262408 & 1.25774 & 1.260816 & 1.269963 & 1.253345 \\ \cline{3-14} 
 &  & cGA(m=4)-FIFO-1 & 1.229064 & 1.218781 & 1.209763 & 1.198236 & 1.235635 & 1.204047 & 1.21714 & 1.243012 & 1.240044 & 1.203013 & 1.219874 \\ \cline{3-14} 
 &  & cGA(m=4)-LRU-1 & 1.257657 & 1.261423 & 1.215217 & 1.211751 & 1.264149 & 1.210676 & 1.232505 & 1.203418 & 1.216826 & 1.214645 & 1.228827 \\ \cline{3-14} 
 &  & pe-cGA-FIFO-1 & 1.295234 & 1.36425 & 1.404062 & 1.451323 & 1.50718 & 1.547016 & 1.48251 & 1.5731 & 1.556566 & 1.577982 & 1.475922 \\ \cline{3-14} 
 &  & pe-cGA-LRU-1 & 1.319367 & 1.339874 & 1.46564 & 1.464298 & 1.445212 & 1.514818 & 1.486544 & 1.565565 & 1.541142 & 1.568835 & 1.47113 \\ \cline{3-14} 
 &  & ne-cGA-FIFO-1 & 1.192062 & 1.232879 & 1.271284 & 1.299795 & 1.316486 & 1.339061 & 1.379007 & 1.432261 & 1.393953 & 1.413282 & 1.327007 \\ \cline{3-14} 
 &  & ne-cGA-LRU-1 & 1.217918 & 1.258857 & 1.287642 & 1.284838 & 1.312982 & 1.345824 & 1.376772 & 1.360448 & 1.432775 & 1.44232 & 1.332038 \\ \cline{2-14} 
 & \multirow{10}{*}{\begin{tabular}[c]{@{}l@{}}Cache \\ Length\\ 20\end{tabular}} & cGA-FIFO-20 & 1.544273 & 1.512646 & 1.453535 & 1.436836 & 1.402085 & 1.432222 & 1.402666 & 1.396865 & 1.403451 & 1.407687 & 1.439227 \\ \cline{3-14} 
 &  & cGA-LRU-20 & 1.58412 & 1.485608 & 1.456194 & 1.447313 & 1.458603 & 1.413239 & 1.440135 & 1.417762 & 1.396936 & 1.427347 & 1.452726 \\ \cline{3-14} 
 &  & cGA(s=4)-FIFO-20 & 1.63062 & 1.650399 & 1.59405 & 1.578122 & 1.575991 & 1.587032 & 1.599778 & 1.559781 & 1.570364 & 1.562222 & 1.590836 \\ \cline{3-14} 
 &  & cGA(s=4)-LRU-20 & 1.732065 & 1.617206 & 1.636188 & 1.564301 & 1.628552 & 1.588255 & 1.563205 & 1.577034 & 1.583888 & 1.60592 & 1.609661 \\ \cline{3-14} 
 &  & cGA(m=4)-FIFO-20 & 1.782668 & 1.603246 & 1.584661 & 1.556634 & 1.481432 & 1.519929 & 1.525975 & 1.521023 & 1.460363 & 1.526995 & 1.556293 \\ \cline{3-14} 
 &  & cGA(m=4)-LRU-20 & 1.600072 & 1.624054 & 1.598647 & 1.622701 & 1.537502 & 1.576676 & 1.543509 & 1.514601 & 1.540614 & 1.518711 & 1.567709 \\ \cline{3-14} 
 &  & pe-cGA-FIFO-20 & 1.797184 & 1.991112 & 2.057194 & 2.099994 & 2.156264 & 2.064703 & 2.24526 & 2.163161 & 2.210175 & 2.279559 & 2.106461 \\ \cline{3-14} 
 &  & pe-cGA-LRU-20 & 1.831156 & 1.954343 & 2.074114 & 2.169858 & 2.126215 & 2.194234 & 2.176215 & 2.341222 & 2.275895 & 2.291566 & 2.143482 \\ \cline{3-14} 
 &  & ne-cGA-FIFO-20 & 1.563835 & 1.615645 & 1.661901 & 1.671523 & 1.699175 & 1.748135 & 1.675225 & 1.742129 & 1.754192 & 1.759943 & 1.68917 \\ \cline{3-14} 
 &  & ne-cGA-LRU-20 & 1.56987 & 1.655711 & 1.717918 & 1.682804 & 1.738595 & 1.797893 & 1.758643 & 1.816472 & 1.835724 & 1.817665 & 1.73913 \\ \hline
\end{tabular}%
}
\end{table*}

In this section, the performance of the proposed cache-based techniques is compared with that of original cGA, pe-cGA, and ne-cGA based on two benchmark functions, defined in Appendix A. Here cGA, pe-cGA, and ne-cGA will be referred to as non-cache-based algorithms. As the convergence and accuracy are the same for both cache-based and non-cache-based algorithms, we discuss and compare our experimental results in terms of the total number of fitness function evaluations. All experimental results are averaged over 50 runs. All runs end when the population fully converges, that is, when for each gene position, all the population members have the same allele value (either 0 or 1). From here onwards, $x$-FIFO and $x$-LRU will be used to denote the cache-based (our proposed) algorithm of algorithm-$x$ using FIFO and LRU as the cache replacement strategies, respectively. For example, cGA-FIFO represents the cache-based version of cGA using FIFO as the cache replacement policy. Similarly, pe-cGA-LRU denotes the cache-based algorithm of pe-cGA using LRU as the cache replacement strategy. Moreover, for convenience, $x$-FIFO-$y$ and $x$-LRU-$y$ are used to denote $x$-FIFO and $x$-LRU, respectively, with cache capacity of $y$. For example, cGA-FIFO-1 denotes cGA with FIFO as the cache replacement strategy for cache capacity 1.

Experiments are also conducted on cGA with higher selection pressures. cGA with a tournament of size $x$, and cGA with round-robin tournament of size $x$, are represented by cGA($s=x$) and cGA($m=x$), respectively. The details of these selection techniques have already been discussed in Section~\ref{sec:cGA}.

\subsection{Experiments with Respect to Different Population Sizes}

\begin{figure*}
     \centering
     \begin{subfigure}{\subplotsize\linewidth}
         \includegraphics[width=\linewidth]{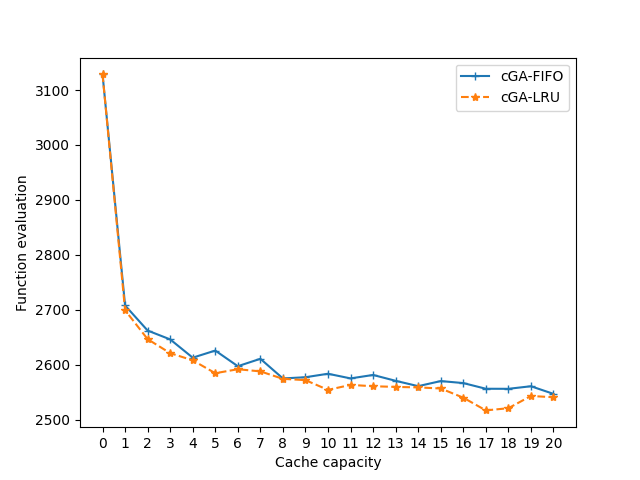}
         \caption{100-bit onemax problem}
         \label{fig:onemax_cGA_eval_lc}
     \end{subfigure}
     \begin{subfigure}{\subplotsize\linewidth}
         \includegraphics[width=\linewidth]{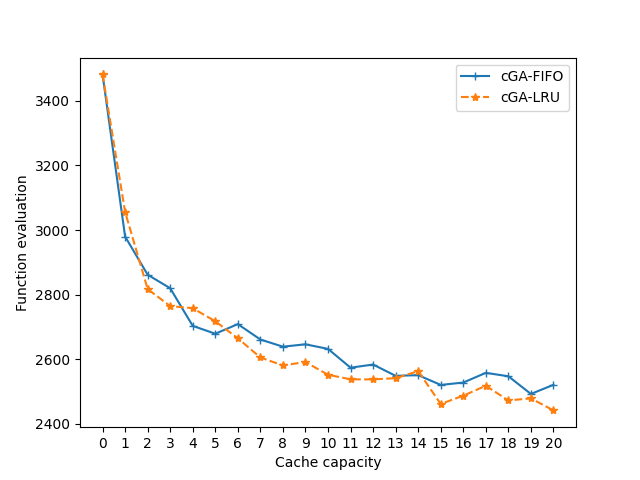}
         \caption{30-bit binary integer problem}
         \label{fig:binint_cGA_eval_lc}
     \end{subfigure}
     \begin{subfigure}{\subplotsize\linewidth}
         \includegraphics[width=\linewidth]{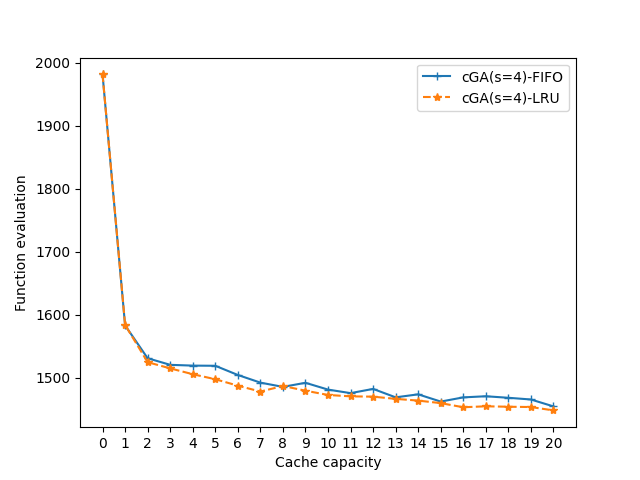}
         \caption{100-bit onemax problem}
         \label{fig:onemax_cGAt4_eval_lc}
     \end{subfigure}
     \begin{subfigure}{\subplotsize\linewidth}
         \includegraphics[width=\linewidth]{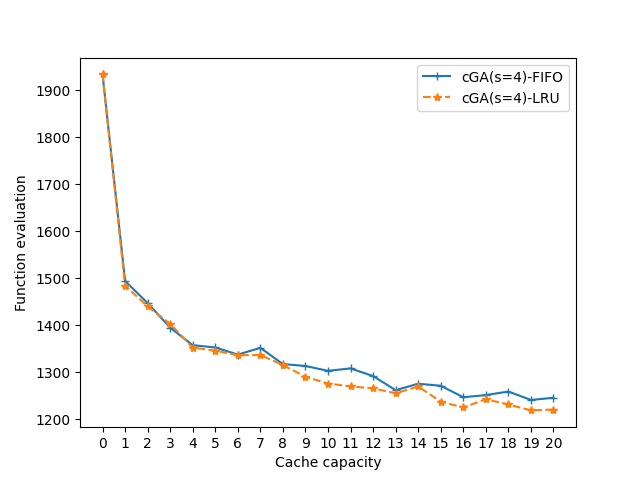}
         \caption{30-bit binary integer problem}
         \label{fig:binint_cGAt4_eval_lc}
     \end{subfigure}
     \begin{subfigure}{\subplotsize\linewidth}
         \includegraphics[width=\linewidth]{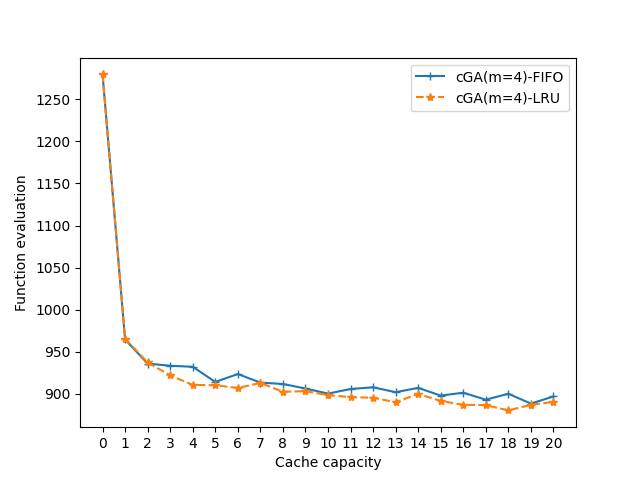}
         \caption{100-bit onemax problem}
         \label{fig:onemax_cGArrt4_eval_lc}
     \end{subfigure}
     \begin{subfigure}{\subplotsize\linewidth}
         \includegraphics[width=\linewidth]{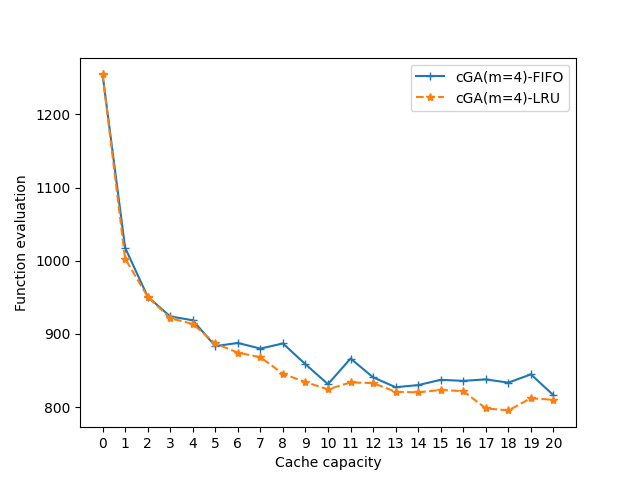}
         \caption{30-bit binary integer}
         \label{fig:binint_cGArrt4_eval_lc}
     \end{subfigure}
     \caption{The plots illustrate the number of fitness function evaluations by cGA and cGA with cache for the 100-bit onemax problem and 30-bit binary integer problem with a population size of 100. The algorithms were executed for cache capacities varying from 0 to 20 with an increment of 1}
        \label{fig:cGA_lc}
\end{figure*}

\begin{figure*}
     \centering
     \begin{subfigure}{\subplotsize\linewidth}
         \includegraphics[width=\linewidth]{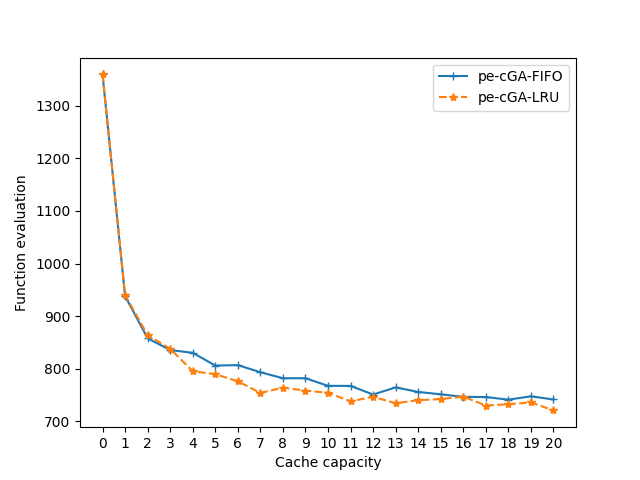}
         \caption{100-bit onemax problem}
         \label{fig:onemax_pecGA_eval_lc}
     \end{subfigure}
     \begin{subfigure}{\subplotsize\linewidth}
         \includegraphics[width=\linewidth]{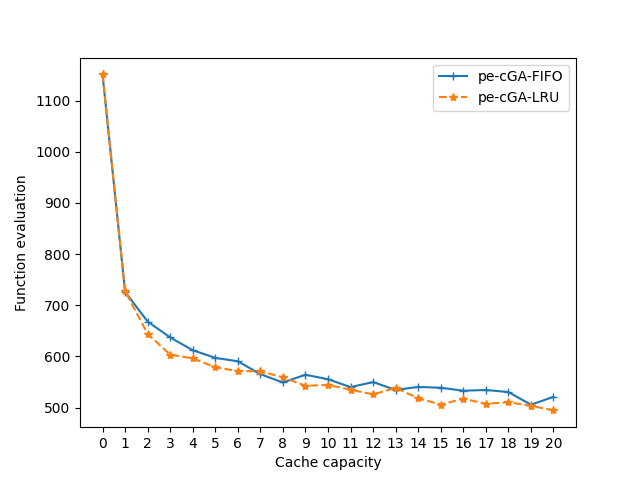}
         \caption{30-bit binary integer problem}
         \label{fig:binint_pecGA_eval_lc}
     \end{subfigure}
     \begin{subfigure}{\subplotsize\linewidth}
         \includegraphics[width=\linewidth]{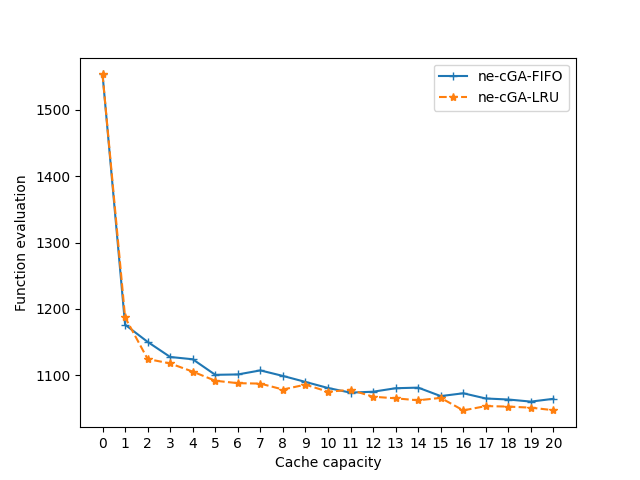}
         \caption{100-bit onemax problem}
         \label{fig:onemax_necGA_eval_lc}
     \end{subfigure}
     \begin{subfigure}{\subplotsize\linewidth}
         \includegraphics[width=\linewidth]{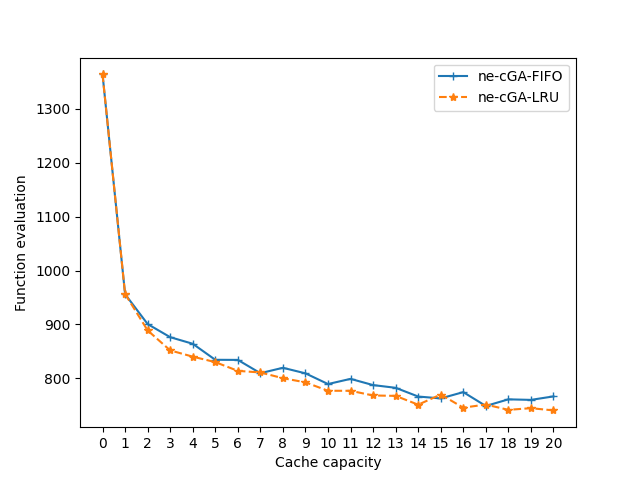}
         \caption{30-bit binary integer problem}
         \label{fig:binint_necGA_eval_lc}
     \end{subfigure}
     \caption{The plots illustrate the number of fitness function evaluations by elitism-based cGA and elitism-based cGA with cache for the 100-bit onemax problem and 30-bit binary integer problem with a population size of 100. The algorithms were executed for cache capacities varying from 0 to 20 with an increment of 1}
        \label{fig:elite_cGA_lc}
\end{figure*}

\begin{figure*}
     \centering
     \begin{subfigure}{\subplotsize\linewidth}
         \includegraphics[width=\linewidth]{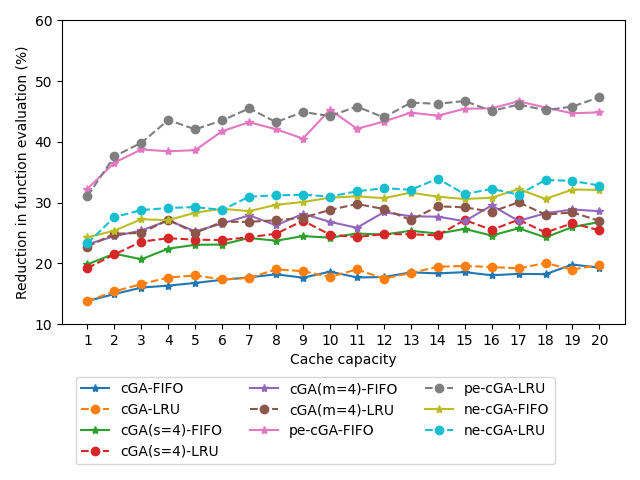}
         \caption{100-bit onemax problem}
         \label{fig:onemax_hit_lc}
     \end{subfigure}
     \begin{subfigure}{\subplotsize\linewidth}
         \includegraphics[width=\linewidth]{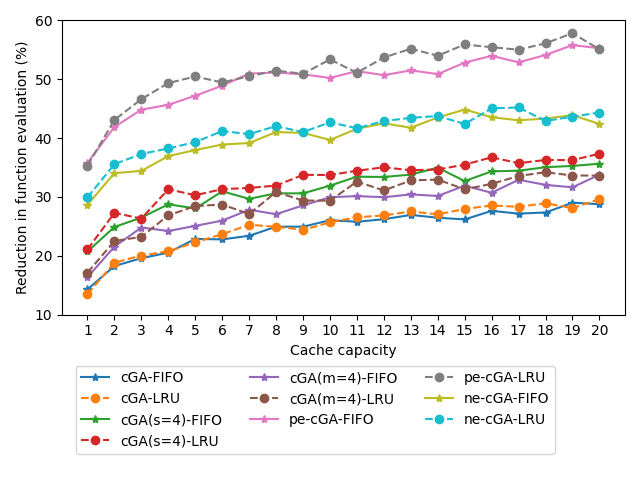}
         \caption{30-bit binary integer problem}
         \label{fig:binint_hit_lc}
     \end{subfigure}
     \caption{The plots illustrate the percentage of reduction in fitness function evaluation due to the use of cache by cGA and elitism-based cGA for the 100-bit onemax problem and 30-bit binary integer problem with a population size of 100. The algorithms were executed for the cache capacity varying from 0 to 20 with an increment of 1}
        \label{fig:hit_lc}
\end{figure*}

\begin{table*}
\centering
\caption{Speedup achieved in terms of the number of fitness function evaluations due to the use of cache by cGA and elitism-based cGA with respect to different cache capacities}
\label{tab:speedup_lc}
\resizebox{\textwidth}{!}{%
\begin{tabular}{|l|l|l|l|l|l|l|l|l|l|l|l|l|l|l|l|l|l|l|l|l|l|l|}
\hline
\multicolumn{2}{|l|}{Cache capacities =\textgreater{}} & 1 & 2 & 3 & 4 & 5 & 6 & 7 & 8 & 9 & 10 & 11 & 12 & 13 & 14 & 15 & 16 & 17 & 18 & 19 & 20 & Avg. \\ \hline
\multirow{10}{*}{\begin{tabular}[c]{@{}l@{}}Onemax\\ Problem\end{tabular}} & cGA-FIFO & 1.16 & 1.18 & 1.19 & 1.2 & 1.2 & 1.21 & 1.22 & 1.23 & 1.22 & 1.23 & 1.22 & 1.22 & 1.23 & 1.23 & 1.23 & 1.22 & 1.23 & 1.22 & 1.25 & 1.24 & 1.22 \\ \cline{2-23} 
 & cGA-LRU & 1.16 & 1.18 & 1.2 & 1.22 & 1.22 & 1.21 & 1.21 & 1.24 & 1.23 & 1.22 & 1.24 & 1.21 & 1.23 & 1.24 & 1.25 & 1.24 & 1.24 & 1.25 & 1.24 & 1.25 & 1.22 \\ \cline{2-23} 
 & cGA(s=4)-FIFO & 1.25 & 1.28 & 1.26 & 1.29 & 1.31 & 1.3 & 1.33 & 1.32 & 1.33 & 1.33 & 1.34 & 1.33 & 1.35 & 1.34 & 1.35 & 1.33 & 1.35 & 1.33 & 1.36 & 1.37 & 1.32 \\ \cline{2-23} 
 & cGA(s=4)-LRU & 1.24 & 1.28 & 1.32 & 1.32 & 1.32 & 1.32 & 1.33 & 1.34 & 1.38 & 1.33 & 1.33 & 1.33 & 1.33 & 1.33 & 1.38 & 1.35 & 1.38 & 1.34 & 1.37 & 1.35 & 1.33 \\ \cline{2-23} 
 & cGA(m=4)-FIFO & 1.31 & 1.33 & 1.35 & 1.39 & 1.34 & 1.37 & 1.41 & 1.36 & 1.4 & 1.38 & 1.36 & 1.41 & 1.4 & 1.4 & 1.38 & 1.43 & 1.38 & 1.41 & 1.42 & 1.41 & 1.38 \\ \cline{2-23} 
 & cGA(m=4)-LRU & 1.31 & 1.34 & 1.34 & 1.39 & 1.35 & 1.38 & 1.38 & 1.38 & 1.39 & 1.42 & 1.44 & 1.42 & 1.39 & 1.43 & 1.42 & 1.41 & 1.44 & 1.4 & 1.41 & 1.38 & 1.39 \\ \cline{2-23} 
 & pe-cGA-FIFO & 1.49 & 1.6 & 1.65 & 1.65 & 1.64 & 1.75 & 1.79 & 1.75 & 1.71 & 1.85 & 1.75 & 1.79 & 1.85 & 1.83 & 1.86 & 1.86 & 1.9 & 1.86 & 1.84 & 1.84 & 1.76 \\ \cline{2-23} 
 & pe-cGA-LRU & 1.46 & 1.63 & 1.69 & 1.8 & 1.76 & 1.8 & 1.86 & 1.79 & 1.85 & 1.82 & 1.88 & 1.82 & 1.91 & 1.89 & 1.91 & 1.85 & 1.89 & 1.85 & 1.87 & 1.94 & 1.81 \\ \cline{2-23} 
 & ne-cGA-FIFO & 1.33 & 1.35 & 1.39 & 1.38 & 1.4 & 1.42 & 1.41 & 1.43 & 1.44 & 1.45 & 1.47 & 1.45 & 1.47 & 1.45 & 1.45 & 1.46 & 1.49 & 1.45 & 1.48 & 1.48 & 1.43 \\ \cline{2-23} 
 & ne-cGA-LRU & 1.31 & 1.39 & 1.41 & 1.42 & 1.43 & 1.41 & 1.46 & 1.46 & 1.46 & 1.46 & 1.48 & 1.49 & 1.48 & 1.52 & 1.46 & 1.48 & 1.46 & 1.52 & 1.52 & 1.5 & 1.46 \\ \hline
\multirow{10}{*}{\begin{tabular}[c]{@{}l@{}}Binary\\ Integer\\ Problem\end{tabular}} & cGA-FIFO & 1.17 & 1.23 & 1.25 & 1.26 & 1.3 & 1.3 & 1.31 & 1.34 & 1.34 & 1.36 & 1.35 & 1.36 & 1.37 & 1.36 & 1.36 & 1.39 & 1.38 & 1.38 & 1.41 & 1.41 & 1.33 \\ \cline{2-23} 
 & cGA-LRU & 1.16 & 1.24 & 1.25 & 1.26 & 1.29 & 1.31 & 1.34 & 1.34 & 1.33 & 1.35 & 1.36 & 1.37 & 1.38 & 1.37 & 1.39 & 1.41 & 1.4 & 1.41 & 1.4 & 1.43 & 1.34 \\ \cline{2-23} 
 & cGA(s=4)-FIFO & 1.27 & 1.34 & 1.37 & 1.41 & 1.4 & 1.46 & 1.43 & 1.45 & 1.45 & 1.48 & 1.51 & 1.51 & 1.52 & 1.55 & 1.5 & 1.53 & 1.53 & 1.55 & 1.56 & 1.56 & 1.47 \\ \cline{2-23} 
 & cGA(s=4)-LRU & 1.27 & 1.38 & 1.37 & 1.47 & 1.44 & 1.46 & 1.47 & 1.48 & 1.52 & 1.52 & 1.53 & 1.55 & 1.54 & 1.54 & 1.57 & 1.6 & 1.56 & 1.58 & 1.58 & 1.61 & 1.5 \\ \cline{2-23} 
 & cGA(m=4)-FIFO & 1.2 & 1.28 & 1.34 & 1.33 & 1.35 & 1.36 & 1.4 & 1.38 & 1.41 & 1.44 & 1.45 & 1.44 & 1.45 & 1.45 & 1.49 & 1.46 & 1.51 & 1.48 & 1.48 & 1.53 & 1.41 \\ \cline{2-23} 
 & cGA(m=4)-LRU & 1.21 & 1.3 & 1.31 & 1.38 & 1.41 & 1.41 & 1.39 & 1.46 & 1.43 & 1.43 & 1.51 & 1.47 & 1.52 & 1.51 & 1.46 & 1.49 & 1.52 & 1.53 & 1.52 & 1.52 & 1.44 \\ \cline{2-23} 
 & pe-cGA-FIFO & 1.58 & 1.76 & 1.84 & 1.88 & 1.93 & 2 & 2.08 & 2.1 & 2.07 & 2.03 & 2.08 & 2.07 & 2.1 & 2.07 & 2.16 & 2.21 & 2.17 & 2.23 & 2.31 & 2.28 & 2.05 \\ \cline{2-23} 
 & pe-cGA-LRU & 1.57 & 1.79 & 1.91 & 2.01 & 2.06 & 2.02 & 2.06 & 2.1 & 2.08 & 2.22 & 2.09 & 2.21 & 2.28 & 2.22 & 2.34 & 2.29 & 2.28 & 2.32 & 2.41 & 2.29 & 2.13 \\ \cline{2-23} 
 & ne-cGA-FIFO & 1.41 & 1.53 & 1.54 & 1.6 & 1.63 & 1.66 & 1.67 & 1.71 & 1.71 & 1.67 & 1.73 & 1.76 & 1.74 & 1.79 & 1.84 & 1.79 & 1.77 & 1.79 & 1.8 & 1.76 & 1.69 \\ \cline{2-23} 
 & ne-cGA-LRU & 1.44 & 1.57 & 1.62 & 1.63 & 1.67 & 1.72 & 1.72 & 1.74 & 1.72 & 1.78 & 1.74 & 1.77 & 1.79 & 1.81 & 1.76 & 1.84 & 1.85 & 1.77 & 1.8 & 1.82 & 1.73 \\ \hline
\end{tabular}%
}
\end{table*}

Fig.~\ref{fig:cGA_n} shows the results of experiments on cGA, cGA with cache having a length of 1 and 20, and with FIFO and LRU as the cache replacement strategies. The subplots in the first column show the experimental results on 100-bit onemax problem, and the subplots in the second column show the experimental results on 30-bit binary integer problem. All the experiments were done by varying the population size from 10 to 100 with step size 10. Fig.~\ref{fig:onemax_cGA_eval_n} and Fig.~\ref{fig:binint_cGA_eval_n} show the number of fitness function evaluations by cGA, cGA-FIFO-1, cGA-LRU-1, cGA-FIFO-20, and cGA-LRU-20. The two figures show that cGA with cache (cGA-FIFO-1, cGA-LRU-1, cGA-FIFO-20, cGA-LRU-20) has a lesser number of fitness function evaluations compared to that of cGA without cache. This is also evident from the gap between the curves. As the population size increases, the gap between them also increases. This is because the probability vector in cGA is updated with the reciprocal value of the population size. As the population size increases, the probability vector's convergence rate slows down. This increases the frequency of cache hits because the chromosomes are generated based on the probability vector. As a larger cache can enhance the locality of reference, the number of fitness function evaluations is much lower for cGA-LRU-20 for both problems. However, one observation here is that even a cache of length 1 can improve the performance significantly. Here the performance of the two cache replacement strategies - FIFO and LRU are comparably similar. LRU performs slightly better, which can be observed from the population size of 100 in Fig.~\ref{fig:binint_cGA_eval_n}. Our proposed method also works well in the case of cGA with tournament selection of size 4 (Fig.~\ref{fig:onemax_cGAt4_eval_n} and Fig.~\ref{fig:binint_cGAt4_eval_n}) and cGA with round-robin tournament selection of size 4 (Fig.~\ref{fig:onemax_cGArrt4_eval_n} and Fig.~\ref{fig:binint_cGArrt4_eval_n}) in terms of number of the fitness function evaluations.

In elitism-based cGA, the number of fitness function evaluations is reduced by preserving a dominant chromosome, called the elite chromosome. Non-persistent elitist cGA replaces the elite chromosome after some predefined number of iterations to avoid the scenario of premature convergence. Fig.~\ref{fig:elite_cGA_n} presents the number of fitness function evaluations by persistent elitist cGA (pe-cGA) and non-persistent elitist cGA (ne-cGA) for 100-bit onemax problem (first column) and 30-bit binary integer problem (second column). Since here the elite chromosome is preserved separately, our proposed cache-based algorithms like pe-cGA-FIFO-1, pe-cGA-LRU-1, pe-cGA-FIFO-20, pe-cGA-LRU-20, ne-cGA-FIFO-1, ne-cGA-LRU-1, ne-cGA-FIFO-20 and ne-cGA-LRU-20 perform better, as is evident from the figure.

Table~\ref{tab:speedup_n} reports the {\em Speedup} achieved by cache-based algorithms over that by the corresponding non-cache-based algorithms in terms of the number of fitness function evaluations. {\em Speedup} is calculated with respect to different population sizes from 10 to 100 with step size 10. The average is also calculated across all the population sizes for a particular algorithm, particular problem, and particular cache capacity. From the table, it is clear that cache-based algorithms perform better than corresponding non-cache-based algorithms. The performance improvement is more significant in the case of pe-CGA-FIFO, pe-cGA-LRU, ne-cGA-FIFO, and ne-cGA-LRU.

Fig.~\ref{fig:hit_n} shows the relative performance improvement. \textit{Percentage of reduction in number of fitness function evaluations} or simply \textit{reduction in function evaluations (\%)} = $\frac{neval-neval_{cache}}{neval} \times 100$, where $neval_{cache}$ and $neval$ are the number of function evaluations with and without cache, respectively. Fig.~\ref{fig:onemax_hit_c1_n} and Fig.~\ref{fig:binint_hit_c1_n} show the reduction in function evaluations (\%) for the cache capacity of 1, whereas, Fig.~\ref{fig:onemax_hit_c20_n} and Fig.~\ref{fig:binint_hit_c20_n} provide the reduction in function evaluations (\%) for the cache capacity of 20. As expected, performance improvement for the cache capacity of 20 is much bigger than that for the cache capacity of 1. However, irrespective of the cache capacity and the problem, a clear hierarchy can be observed in the results. In elitism-based cGA with cache (pe-cGA-FIFO, pe-cGA-LRU, ne-cGA-FIFO, and ne-cGA-LRU), the elite chromosome influences the probability vector. Therefore, the highest performance improvement can be observed here. Then comes the performance improvement by cGA with round-robin tournament selection and tournament selection due to their larger selection pressure. In the original cGA, as the tournament size is 2, it comes last in the hierarchy.

\subsection{Experiments with Respect to Different Cache Capacities}

Fig.~\ref{fig:cGA_lc} and Fig.~\ref{fig:elite_cGA_lc} illustrate the number of function evaluations by different algorithms with respect to cache capacities from 0 to 20 with a step size of 1. Here cache capacity 0 means the algorithm is not using any cache. From the subfigures, it is clear that a cache capacity of 1 reduces the fitness function evaluations significantly. As the cache capacity increases, the number of fitness function evaluations reduces. After a certain cache capacity, saturation can be observed in function evaluations. This means function evaluation does not reduce significantly with the increase of cache capacity. The scenario can be easily understood from Fig.~\ref{fig:hit_lc} which shows the percentage of reduction in the number of fitness function evaluations with respect to cache capacity. Table~\ref{tab:speedup_lc} also shows the speedup achieved by cache-based algorithms than the corresponding non-cache-based algorithms. 

\section{Conclusion}\label{conclude}
In this article, we have proposed an improved cGA with an efficient caching mechanism. The proposed technique uses the concept of locality of reference while evaluating new chromosomes. The experimental results have shown that cGA and the elitism-based cGA with caching perform far better than the original cGA and elitism-based cGA in terms of the number of function evaluations. It reduces the number of fitness function evaluations significantly, even when the cache capacity of only 1. We have explained the performance of our proposed method based on different metrics that are independent of the objective function. We have also presented a suitable data structure for maintaining the cache with minimum system overhead and in asymptotically constant time complexity on average while retaining the same accuracy and convergence rate of the respective algorithms. The proposed cache-based algorithms are applicable to all the problems where cGA and elitism-based cGAs are applicable. This will be particularly useful for the problems for which fitness computation is a time/space consuming procedure.

\appendices
\section{Problems}
This appendix presents some problems, based on which our experiments are performed.
\subsection*{Onemax Problem}
It is a maximization problem. For a $l$-bit onemax problem, a string of length $l$ is considered. The objective is to find a string of length $l$ with a maximum number of 1s. The fitness value of a string $x$ is evaluated as the number of 1s in that string. The strings and fitnesses of 2-bit onemax problem are shown below: 
\begin{center}
\begin{tabular}{ c c }
\hline
 String & Fitness \\\hline
 00 & 0 \\
 01 & 1 \\
 10 & 1 \\
 11 & 2\\\hline
\end{tabular}
\end{center}
\subsection*{Binary Integer Problem}
It is a maximization problem. For a $l$-bit binary integer problem, a $l$-bit binary string is considered. The fitness of a binary string is the equivalent decimal value of that string. The objective is to find a $l$-bit binary string with maximum decimal value. The strings and fitnesses of 2-bit binary integer problem with fitnesses are shown below: 
\begin{center}
\begin{tabular}{ c c }
\hline
 String & Fitness \\\hline
 00 & 0 \\
 01 & 1 \\
 10 & 2 \\
 11 & 3\\\hline
\end{tabular}
\end{center}
\bibliographystyle{IEEEtran}
\bibliography{references}
\end{document}